\documentclass[a4paper,10pt]{amsart}

\usepackage{amsmath,amssymb,amsthm,a4wide}

\usepackage{tikz}
\usetikzlibrary{shapes}
\usetikzlibrary{intersections}
\usepackage{pgfplots}
\usepackage{mdframed}
\usepackage{graphicx}
\usepackage{graphics}
\usepackage{caption}

\newtheorem*{thm}{Theorem}

\def\nodeDist{1.2cm}
\tikzstyle{origVertex}      = [draw, black, fill, shape=circle]
\tikzstyle{newVertex}       = [draw, black,  fill, shape=circle]
\tikzstyle{invisibleVertex} = [shape=circle]
\tikzstyle{origEdge}      = [black]
\tikzstyle{newEdge}       = [black]
\tikzstyle{invisibleEdge} = [draw opacity=0]

\begin{document}

\title[Spectral Echolocation embedding]{Spectral Echolocation via the Wave Embedding}
\author{Alexander Cloninger}
\address[Alexander Cloninger]{Department of Mathematics, Program in Applied Mathematics, Yale University, New Haven, CT 06510, USA}
\email{alexander.cloninger@yale.edu}

\author{Stefan Steinerberger}
\address[Stefan Steinerberger]{Department of Mathematics, Yale University, 06511 New Haven, CT, USA}
\email{stefan.steinerberger@yale.edu}

\begin{abstract} Spectral embedding uses eigenfunctions of the discrete Laplacian on a weighted graph to obtain coordinates for an embedding of an abstract
data set into Euclidean space. We propose a new pre-processing step of first using the eigenfunctions to simulate a low-frequency wave moving over the data and using both position as well
as change in time of the wave to obtain a refined metric to which classical methods of dimensionality reduction can then applied. This is motivated by
the behavior of waves, symmetries of the wave equation and the hunting technique of bats. It is shown to be effective in practice and also works for
other partial differential equations -- the method yields improved results even for the classical heat equation.

\end{abstract}
\maketitle

\section{Introduction}
Spectral embedding methods are based on analyzing Markov chains on a high-dimensional data set $\left\{x_i\right\}_{i=1}^{n} \subset \mathbb{R}^d$. There are a variety of different methods, see e.g. Belkin \& Niyogi \cite{belk}, Coifman \& Lafon \cite{coif1}, Coifman \& Maggioni \cite{coif2}, Donoho \& Grimes \cite{donoho},
Roweis \& Saul \cite{rs},  Tenenbaum, de Silva \& Langford \cite{ten}, and Sahai, Speranzon \& Banaszuk \cite{sahai}. A canonical choice for the weights of the graph is
declare that the probability $p_{ij}$ to move from point $x_j$ to $x_i$ to be
$$ p_{ij} =  \frac{ \exp\left(-\frac{1}{\varepsilon}\|x_i - x_j\|^2_{\ell^2(\mathbb{R}^d)}\right)}{\sum_{k=1}^{n}{ \exp\left(-\frac{1}{\varepsilon}\|x_k - x_j\|^2_{\ell^2(\mathbb{R}^d)}\right)}},$$
where $\varepsilon > 0$ is a parameter that needs to be suitably chosen. This Markov chain can also be interpreted as a weighted
graph that arises as the natural discretization of the underlying 'data-manifold'. Seminal results of Jones, Maggioni \& Schul \cite{jones} justify considering the solutions of
$$ -\Delta \phi_n = \lambda_n^2 \phi_n$$
as measuring the intrinsic geometry of the weighted graph. Here we always assume Neumann-boundary conditions whenever such a graph approximates a manifold.

\begin{figure}[h!]
\begin{tikzpicture}[scale=0.2\textwidth, node distance=\nodeDist,semithick]
 \node[origVertex] (0)              {};
 \node[origVertex] (1) [right of=0] {};
 \node[origVertex] (2) [above of=0] {};
 \node[origVertex] (3) [above of=1] {};
 \node[origVertex] (4) [above of=2] {};
 \path (0) edge[origEdge, out=-45, in=-135]               node[newVertex] (m0) {} (1)
           edge[origEdge, out= 45, in= 135]               node[newVertex] (m1) {} (1)
           edge[origEdge]                                 node[newVertex] (m2) {} (2)
       (1) edge[origEdge]                                 node[newVertex] (m3) {} (3)
       (2) edge[origEdge]                                 node[newVertex] (m4) {} (3)
           edge[origEdge]                                 node[newVertex] (m5) {} (4)
       (3) edge[origEdge, out=125, in=  55, looseness=30] node[newVertex] (m6) {} (3);
 \path (m0) edge[newEdge, out= 135, in=-135]                (m1)
           edge[newEdge, out=  45, in= -45]                (m1)
            edge[newEdge, out=-145, in=-135, looseness=1.7] (m2)
            edge[newEdge, out= -35, in= -45, looseness=1.7] (m3)
       (m1) edge[newEdge]                                   (m2)
            edge[newEdge]                                   (m3)
       (m2) edge[newEdge]                                   (m4)
            edge[newEdge, out= 135, in=-135]                (m5)
       (m3) edge[newEdge]                                   (m4)
            edge[newEdge, out=  45, in=  15]                (m6)
       (m4) edge[newEdge]                                   (m5)
            edge[newEdge, out=  90, in= 165]                (m6) ;
\draw [thick, xshift=0.006cm,yshift=0.005cm] plot [smooth, tension=1] coordinates { (0.03,0.01) (0.04,-0.01)  (0.06,0.01)  (0.055,0.02) (0.05, 0.01) (0.04, 0.01) (0.035, 0.01) (0.03, 0.02) (0.03,0.01) };
\end{tikzpicture}
\caption{Graphs that approximate smooth manifolds.}
\end{figure}
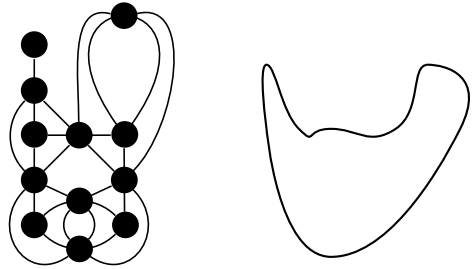

The cornerstone of spectral embedding is the realization that the map
\begin{align*}
\Phi: \left\{x_i\right\}_{i=1}^{n} &\rightarrow \mathbb{R}^k \\
x &\rightarrow  (\phi_1(x), \phi_2(x), \dots, \phi_k(x)).
\end{align*}
can be used as an effective way of reducing the dimensionality. One useful explanation that is often given is to observe that the Feynman-Kac formula
establishes a link between random walks on the weighted graph and the evolution of the heat equation. We observe that random walks have a tendency
to be trapped in clusters and are unlikely to cross over bottlenecks and, simultaneously, that the evolution of the heat equation can be explicitely given as
$$   \left[e^{t \Delta} f\right](x) = \sum_{n=1}^{\infty}{e^{-\lambda_n^2 t} \left\langle f, \phi_n \right\rangle \phi_n(x)}.$$
The exponential decay $e^{-\lambda_n^2 t}$ implies that the long-time dynamics is really governed by the low-lying eigenfunctions who then have
to be able to somehow reconstruct the random walks' inclination for getting trapped in clusters and should thus be able to reconstruct the cluster.
We believe this intuition to be useful and our further exposition will be based on this.

\section{The Wave equation}

\subsection{Introduction.} 

Once the eigenfunctions of the Laplacian have been understood, they imply complete control over the Cauchy problem for the wave equation
\begin{align*}
(\partial_t^2 - \Delta)u(x,t) = 0 \\
u(x,0) = f(x) \\
\partial_t u(x,0) = g(x)
\end{align*}
given by the eigenfunction expansion
$$ u(t,x) = \sum_{n = 1}^{\infty}{ \left[ \cos{(\lambda_ n t)} \left\langle f, \phi_n \right\rangle + \sin{(\lambda_ n t)} \left\langle g, \phi_n \right\rangle \right] \phi_n(x) }.$$
Throughout the rest of the paper, we will understand a solution of a wave equation as an operator of that form, which is meaningful on both smooth, compact manifolds equipped with
the Laplace-Beltrami operator $\Delta_g$ as well as on discrete weighted graphs equipped with the Graph Laplacian $\mathcal{L}$.
A notable difference is the lack of decay associated with the contribution coming from higher eigenfunctions -- this is closely related to the fact that the heat equation is highly smoothing
while the wave equation merely preserves regularity. In one dimension, this is is easily seen using
$$ (\partial_t^2 - \partial_x^2) = (\partial_t - \partial_x)( \partial_t + \partial_x)$$
implying that translations $u(x,t) = f(x+t)$ and $u(x,t) = f(x-t)$ are particular solutions of the wave equation which preserve their initial roughness).
However, the dynamics is still controlled by low-lying eigenfunctions in a time-averaged sense: note
that 
$$ \frac{1}{b-a}\int_{a}^{b}{u(t,x) dt} =  \sum_{n = 1}^{\infty}{ \left[  \left( \frac{1}{b-a} \int_{a}^{b}{\cos{(\lambda_ n t)} dt}\right) \left\langle f, \phi_n \right\rangle +  \left( \frac{1}{b-a} \int_{a}^{b}{ \sin{(\lambda_ n t)} dt}\right) \left\langle g, \phi_n \right\rangle \right] \phi_n(x) }$$ 
where the integrals decay as soon as $\lambda_n^{-1} \lesssim b-a$ since
$$   \frac{1}{b-a} \int_{a}^{b}{ \sin{(\lambda_ n t)} dt} \lesssim \min  \left( 1, \frac{1}{\lambda_n}  \frac{1}{b-a} \right).$$
Put differently, the average behavior over a certain time interval is much smoother than the instantenous behavior.
We will now prove that 'average' considerations within the framework of the wave equation allow us to reconstruct the
classical distance used in spectral embedding: then, after seeing that 'average' considerations recover the known framework,
we will investigate the behavior on shorter time-scales and use that as a way of deriving a finer approximation of
the underlying geometry of the given data.

\subsection{Recovering the spectral distance.}
We start by defining the usual spectral distance between two elements $x_0, y_0 \in \mathcal{M}$ w.r.t. the first $N$ eigenfunctions as
$$ d_N(x_0, y_0)^2 = \sum_{n=1}^{N}{(\phi_n(x_0) - \phi_n(y_0))^2}.$$
Equivalently, this may be understood as the Euclidean distance of the embedding
$$ d_N(x_0, y_0)^2 = \|\Phi_N(x_0) - \Phi_N(y_0)\|_{\ell^2(\mathbb{R}^N)}^2.$$
Given the dynamical setup of a wave equation, there is another natural way of measuring
distances. Given a point $z \in \mathcal{M}$, we define $u_{z}(x,t)$ as the solution of 
\begin{align*}
(\partial_t^2 - \Delta)u_{z}(x,t) &= 0 \\
u_{z}(x,0) &= \delta_{z} \\
\partial_t u_{z}(x,0) &= 0,
\end{align*}
where $\delta_z$ is the Dirac $\delta-$function in the point $z$. The solution starts out being centered at $z$ and then evolves naturally according
to the heat equation. Since we are mainly interested in computational aspects, we will use $u_{z,N}$ to denote the projection of $u_z$ onto the
first $N$ Laplacian eigenfunctions. It is natural to assume that if $x_0, y_0 \in \mathcal{M}$ are close, then $u_{x_0}(x,t)$ and $u_{y_0}(x,t)$ 
should be fairly similar on most points of the domain for most of the time.

\begin{figure}[h!]
\minipage{0.5\textwidth}
\begin{center}
\includegraphics[width=0.9\textwidth]{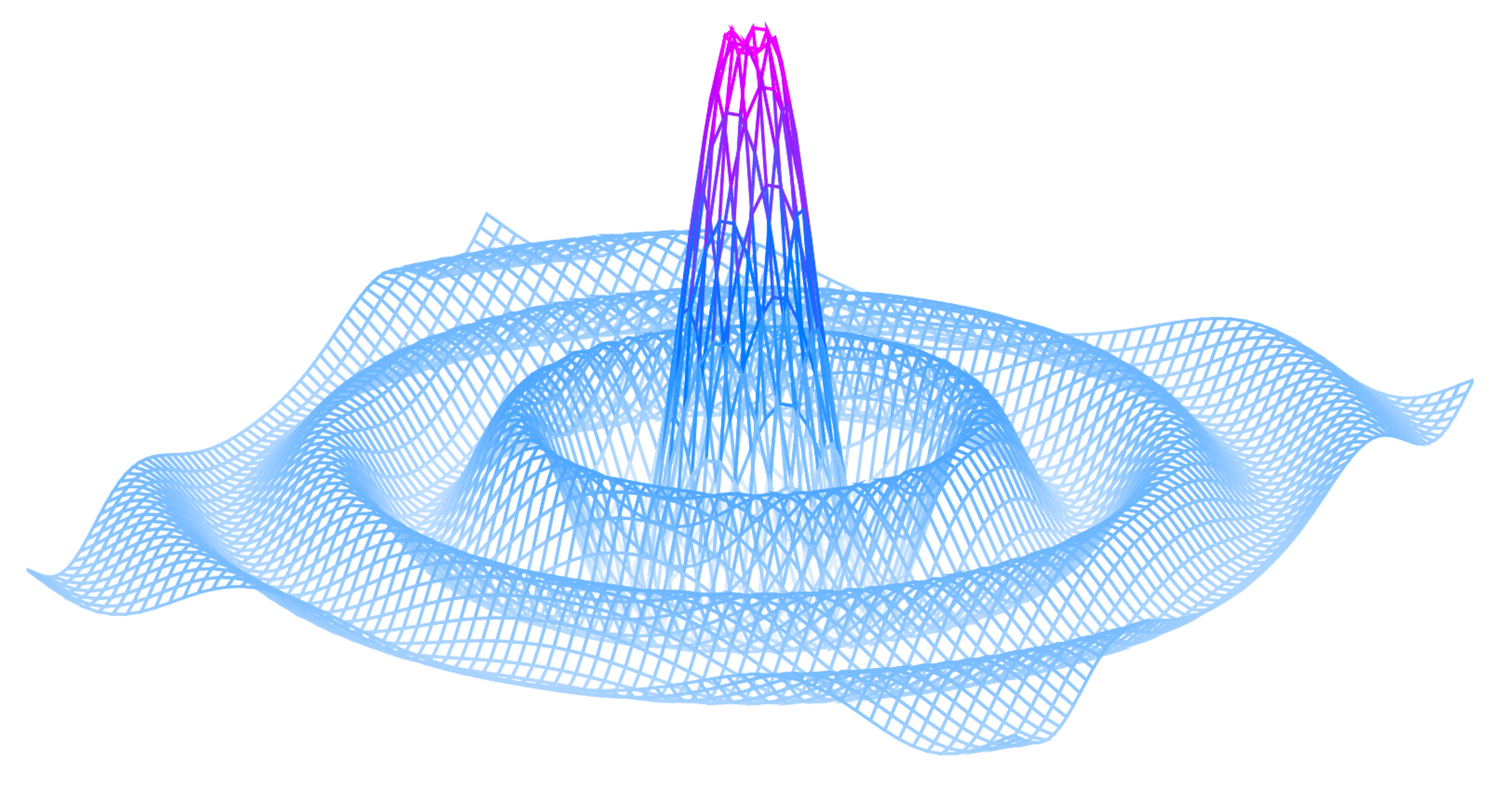}
\end{center}
\endminipage\hfill
\minipage{0.45\textwidth}
\begin{center}
\begin{tikzpicture}[scale=1.2]
\draw [ultra thick] (0,0) circle (0.02cm);
\draw [ultra thick] (0,0) circle (0.2cm);
\draw [ultra thick] (0,0) circle (0.4cm);
\draw [ultra thick] (0,0) circle (0.6cm);
\draw [ultra thick] (0,0) circle (0.8cm);
\draw [ultra thick] (0,0) circle (1cm);
\draw [ultra thick] (0,0) circle (1.2cm);
\draw [ultra thick] (0,0) circle (1.4cm);

\draw [ultra thick] (1.3,0) circle (0.02cm);
\draw [ultra thick] (1.3,0) circle (0.2cm);
\draw [ultra thick] (1.3,0) circle (0.4cm);
\draw [ultra thick] (1.3,0) circle (0.6cm);
\draw [ultra thick] (1.3,0) circle (0.8cm);
\draw [ultra thick] (1.3,0) circle (1cm);
\draw [ultra thick] (1.3,0) circle (1.2cm);
\draw [ultra thick] (1.3,0) circle (1.4cm);
\end{tikzpicture}
\end{center}
\endminipage\hfill
\captionsetup{width=0.9\textwidth}
\caption{An evolving Dirac $\delta-$function and the overlap between two solutions.}
\end{figure}
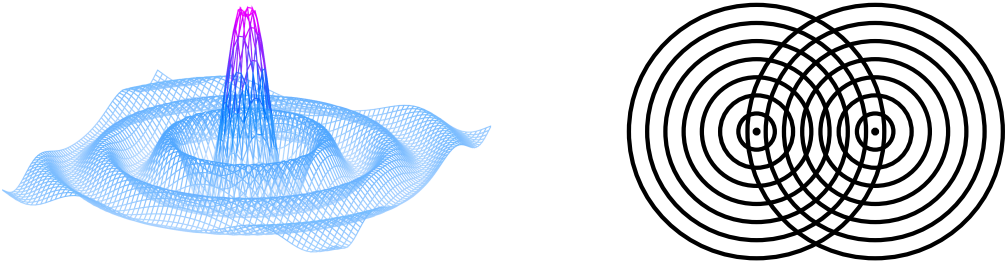

We will now prove the Main Theorem stating that this notion fully recovers the spectral distance.
\begin{thm}[Wave equation can recover spectral distance.] Assume $\mathcal{M}$ is connected (in the sense of $\lambda_1 > 0$). Then the average $L^2-$distance of the wave equation arising from Dirac measures placed in $x_0, y_0$ allows
to reconstruct the spectral distance $d_N(x_0,y_0)$ via
$$ \lim_{T \rightarrow \infty}{\frac{1}{T} \int_{0}^{T}{ \int_{\Omega}{(u_{x_0,N}(x,t) - u_{x_0,N}(x,t))^2 dx} dt}} = \frac{1}{2} d_N(x_0, y_0)^2.$$
\end{thm}
\begin{proof} By definition, we have that
$$ u_{N,x_0}(x,t) =  \sum_{n = 0}^{N}{\cos{(\lambda_ n t)} \left\langle \phi_n, \delta_{x_0} \right\rangle \phi_n(x) } \quad \mbox{and} \quad u_{N,y_0}(x,t) =  \sum_{n = 0}^{N}{\cos{(\lambda_ n t)} \left\langle \phi_n, \delta_{y_0} \right\rangle \phi_n(x) }  $$
We explicitly have that 
$$ \phi_0(x) = \frac{1}{\sqrt{|\mathcal{M}|}} \qquad \mbox{and} \quad  \left\langle \phi_n, \delta_{z} \right\rangle = \phi_n(z).$$
Since the $\phi_n$ are orthonormal in $L^2(\Omega)$, the Pythagorean theorem applies and
\begin{align*}
\int_{\Omega}{(u_N(x,t) - w_N(x,t))^2 dx} &= \int_{\Omega}{ \left( \sum_{n = 1}^{N}{ \cos{(\lambda_ n t)} (\phi_n(x_0) - \phi_n(y_0)) \phi_n(x)  } \right)^2   dx}\\
&= \sum_{n=1}^{N}{\cos{(\lambda_ n t)}^2 (\phi_n(x_0) - \phi_n(y_0))^2   \int_{\Omega}{\phi_n(x)^2 dx}}\\
&= \sum_{n=1}^{N}{\cos{(\lambda_ n t)}^2 (\phi_n(x_0) - \phi_n(y_0))^2}
\end{align*}
and, since $\lambda_k \geq \lambda_1 > 0$, we easily see that
$$  \lim_{T \rightarrow \infty}{\frac{1}{T} \int_{0}^{T}{\cos{(\lambda_n z)}^2dz}} = \frac{1}{2}$$
and therefore
$$ \lim_{T \rightarrow \infty}{\frac{1}{T} \int_{0}^{T}{ \sum_{n=1}^{N}{\cos{(\lambda_ n t)}^2 (\phi_n(x_0) - \phi_n(y_0))^2} dt}} = \frac{1}{2} \sum_{n=1}^{N}{ (\phi_n(x_0) - \phi_n(y_0))^2}.$$
\end{proof}
\textit{Remark.} If $\mathcal{M}$ is not connected but has multiple connected components, then the argument shows
$$    \frac{1}{2} d_N(x_0, y_0)^2 \leq \lim_{T \rightarrow \infty}{\frac{1}{T} \int_{0}^{T}{ \int_{\Omega}{(u_{x_0,N}(x,t) - u_{x_0,N}(x,t))^2 dx} dt}}  \leq d_N(x_0, y_0)^2.$$

\section{The Algorithm}
\begin{flushright}
\textit{If you want to see something, you send waves in its general direction, you don't throw heat at it.}\\
--  attributed to Peter Lax 
\end{flushright}

\subsection{Spectral Echolocation.} The Theorem discussed in the preceeding section suggests that we lose no information when using distances induced by the wave equation. The main underlying
idea of our approach is that we naturally obtain \textit{additional} information. We emphasize that the algorithm we describe here is \textit{not} a dimension-reduction
algorithm -- instead, it can be regarded as a natural pre-processing step to enhance the effectiveness of spectral methods. Furthermore, it is more appropriate to speak
of an entire family of algorithms: there are a variety of parameters and norms one could define and the optimal choice is not a priori clear.\\

\begin{mdframed}
\textbf{Spectral Echolocation Algorithm.}
\begin{quote}
\begin{enumerate} 
\item \textbf{Input.}~A weighted graph $G = (V,E,w)$.
\item Compute the first $N$ Laplacian eigenfunctions $\left\{\phi_1, \phi_2, \dots, \phi_N\right\}$.
\item Pick $k$ random points $v_1, v_2, \dots, v_k \in V$.
\item Define $k$ functions $f_1, \dots, f_k: V \rightarrow \mathbb{R}$ as 'mollifications' of
the indicator functions associated to the $k$ points. We propose taking the existing affinities given by the weights
$$ f_i(x) = w_{i,x}. $$
\item Pick $\varepsilon > 0$. The projection of the solution of the attenuated wave equation with $f_i$ as initial datum onto the first $N$ eigenfunctions is
$$ u_i(t,x) = \sum_{n=1}^{N}{  \cos{(\lambda_k t)} e^{-\varepsilon \lambda_k t}  \left\langle f_i, \phi_n\right\rangle\phi_n(x)}$$
\item  Define a new weight between any two points $v_1, v_2 \in V$ given by
$$ d_i(v_1, v_2) =\| u_i(t,v_1) - u_i(t,v_2)\|^{\alpha}_{X} + \| (u_i)_t(t,v_1) - (u_i)_t(t,v_2)\|^{\beta}_{Y},$$
where $u_t$ is the derivative in time and $X,Y$ are any norms on the space of continuous functions $C[0,T]$ and $\alpha, \beta > 0$.
\item \textbf{Output.} A distance $d:V \times V \rightarrow \mathbb{R}_{+}$ synthesized out of $d_1, \dots, d_k$, examples being
$$d(v_1,v_2) = \min_{1 \leq i \leq k}{ d_i(v_1, v_2)}  \quad \mbox{and} \quad d(v_1,v_2) = \frac{1}{k}\sum_{i=1}^{k}{ d_i(v_1, v_2)}.$$
\end{enumerate} 
\end{quote}
\end{mdframed}
\vspace{10pt}
The underlying idea is quite simple: we start with various initial distributions of 'water' at rest. We want these initial configurations to be
relatively smooth so as to avoid drastic shocks. Given this initial configuration, we follow the evolution of the wave equation
at our desired level of resolution (given by restricting to $N$ eigenfunctions). Points that are nearby should always have
comparable levels of water and comparable levels of change in water level and this is measured by the integral norm.  
The exponentially decaying term $\exp(-\varepsilon \lambda_k t)$ in the evolution of the solution
$$ u_i(t,x) = \sum_{n=1}^{N}{  \cos{(\lambda_k t)} e^{-\varepsilon \lambda_k t}  \left\langle f_i, \phi_n\right\rangle \phi_n(x)}$$
comes from actually solving for the attenuated wave equation which further reduces high-frequency shocks and increases stability.
As described above, setting $X = L^2$, $\varepsilon = 0$, squaring the norm, ignoring the derivative term completely and letting $T \rightarrow \infty$ recovers
the original weights of the graph completely. In practice, we have found that $T = 1, \varepsilon = \lambda_1^{-1}$, $X = Y = L^1[0,1]$ and $\alpha=\beta=1$ yield
the best results, however, this is a purely experimental finding -- identifying the best parameters and giving a theoretical justification
for their success is still an open problem.

\section{Examples of Noisy Clustering}

\subsection{Parameters} We always consider $N=50$ eigenfunctions and $k=10$ randomly chosen initial spots from which to send out waves. The attenuation factor
is always chosen as $\varepsilon = \lambda_1^{-1}$ and time is chosen so that the first eigenfunction performs one oscillation $T=\lambda_1^{-1}$. Further parameters
are $X = L^1[0,T] = Y$ and $\alpha = 1 = \beta$. This uniquely defines the $1 \leq i \leq k$ individually induced distances, we always condense them into one distance
using either 
$$d(v_1,v_2) = \min_{1 \leq i \leq k}{ d_i(v_1, v_2)}  \qquad \mbox{or} \qquad d(v_1,v_2) = \frac{1}{k}\sum_{i=1}^{k}{ d_i(v_1, v_2)}.$$
Generally, continuous geometries benefit from taking the minimum because of increased smoothness whereas clustering problems are better treated using the second type of combined distance.

\subsection{Geometric Clusters with Erroneous Edges}\label{geomClusterErroneousEdge}

A benefit of the refined wave echolocation metric is that, unlike heat, the transmission between two clusters does not simply depend on the number of edges but also their topology. We consider two clusters in $\mathbb{R}^2$ each of which consists of a 1000 points arranged in a unit disk and the two unit disks are well-separated -- the obstruction comes from a large number of random edges; specifically, every point is randomly connected to 4\% in the other cluster.  Heat diffuses quickly between these two clusters due to the large number of intra cluster connections.  For this reason, the heat embedding of the fails to separate the clusters (however, it is does preserve some aspects of the topology, see Fig. \ref{fig:randErrorEmbeddings}).  In contrast, however, the wave echolocation metric manages 
a clear separation of objects.

\begin{figure}[h!]
\begin{tabular}{cc}
\includegraphics[width=.48\textwidth]{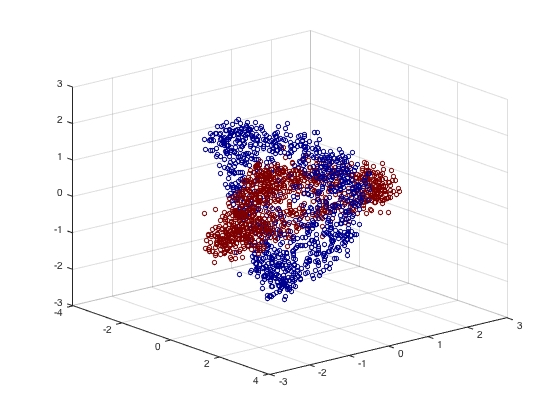} & 
\includegraphics[width=.48\textwidth]{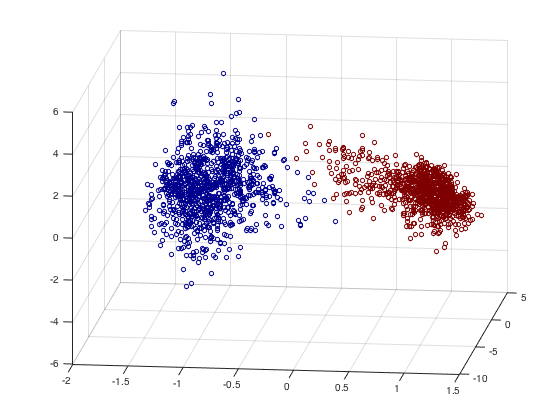} 
\end{tabular}
\caption{Heat kernel embedding (left) and embedding of the wave echolocation metric (right).  We used averaging across 10 starting positions.}\label{fig:randErrorEmbeddings}
\end{figure}

\subsection{Social Networks}
Social networks are a natural application of spectral methods, and mirror the synthetic example in Section \ref{geomClusterErroneousEdge}.  We examine spectral echolocation on the Facebook social circles dataset from \cite{facebookNetwork}, which consists of 4039 people in 10 friend groups.   While there exist clear friend groups, edges within the clusters are still somewhat sparse, and there exist erroneous edges between clusters.  
One goal is to propagate friendship throughout the network and emphasize clusters.  Figure \ref{fig:socialNetworkAffinity} shows the original affinity matrix, sorted by cluster number.  We also compute the diffusion distance and spectral echolocation distance, and display the affinity matrix $W_{i,j} = exp(-d(v_i, v_j)^2)$ for both.  Spectral echolocation not only compresses the inter cluster distances, it also discovers weak similarity between different clusters that share a number of connections.  

\begin{figure}[!h]
\begin{tabular}{ccc}
\includegraphics[width=.32\textwidth]{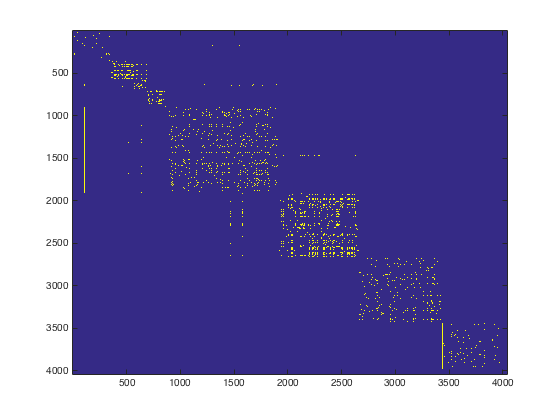} &
\includegraphics[width=.32\textwidth]{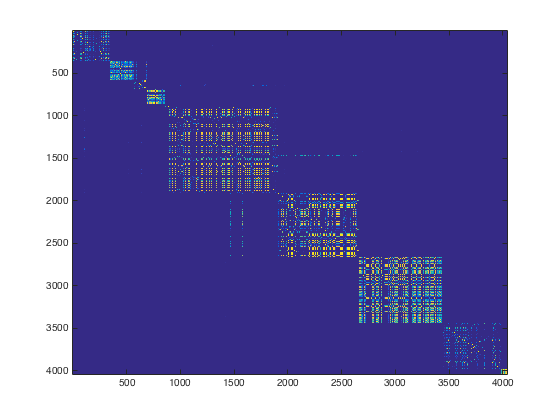} &
\includegraphics[width=.32\textwidth]{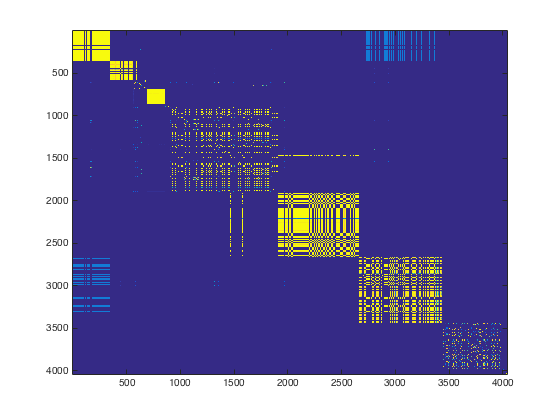} \\
\includegraphics[width=.32\textwidth]{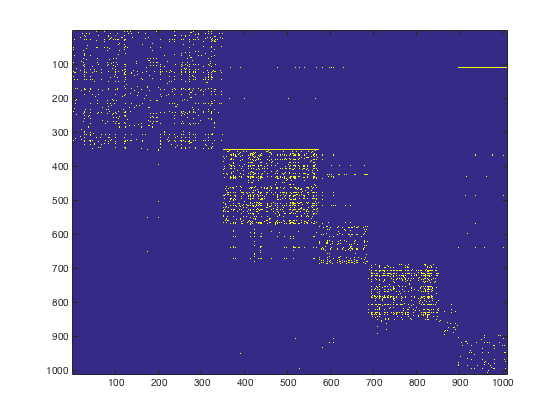} &
\includegraphics[width=.32\textwidth]{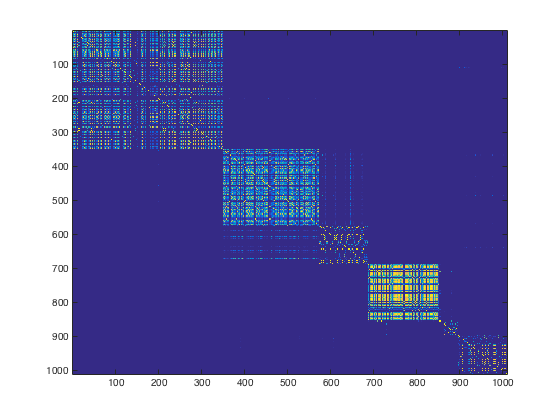} &
\includegraphics[width=.32\textwidth]{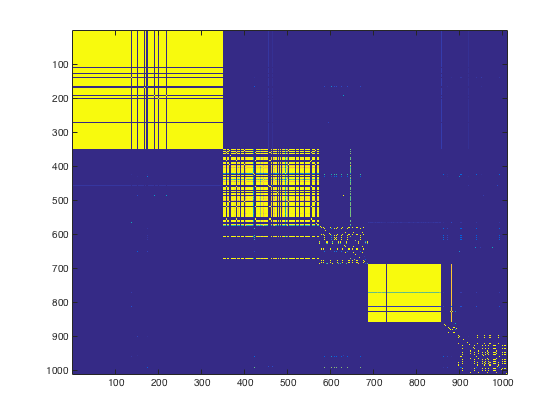} \\
\end{tabular}
\captionsetup{width=0.95\textwidth}
\caption{Original Facebook affinity matrix (left), heat kernel embedding affinity (center), and wave embedding affinity (right).  Bottom is zoomed in version of the top.}\label{fig:socialNetworkAffinity}
\end{figure}

\textit{Substructures.} Another natural goal  is to detect small ``friendship circles'' within the larger network.  These circles are based off of the features that brought the group together (e.g. same university, same work, etc).  Overall there are 193 circles, though many only contain two or three people, and many of the larger circles are nowhere close to a dense network.  We compare the average affinity within the 100 largest circles across several techniques.  For both the standard heat kernel embedding as well as the wave embedding, we build a new graph between people based on whether two people are ``10-times closer than chance'', which is defined as
\begin{eqnarray*}
\{x,y\}\in E \iff e^{-d(x,y)^2/\epsilon} > 10\cdot \mathbb{E}_{x'}\mathbb{E}_{y'} [e^{-d(x',y')^2/\epsilon}].
\end{eqnarray*}
We now compare the typical number of edges in each circle for the original data as well as the two embeddings -- we observe a dramatic improvement. The results are displayed in Figure \ref{fig:circles}.

\begin{figure}[h!]
\begin{tabular}{ccc}
\includegraphics[width=.32\textwidth]{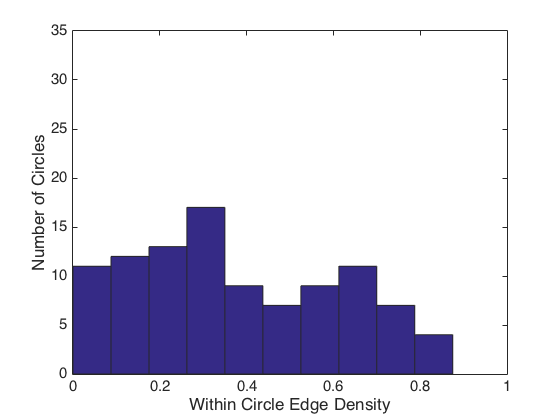} & 
\includegraphics[width=.32\textwidth]{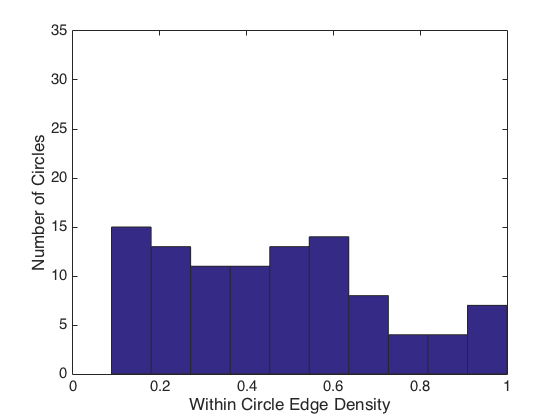} & 
\includegraphics[width=.32\textwidth]{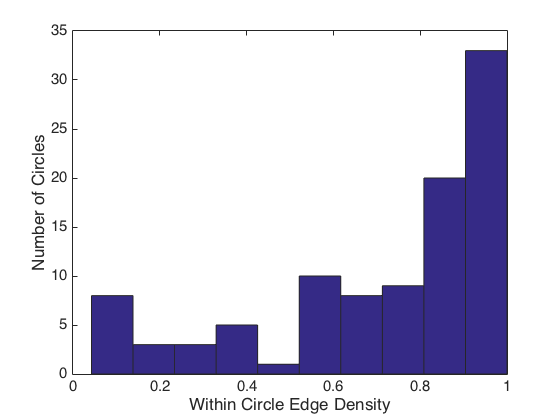} 
\end{tabular}
\caption{Number of friendship edges in original network (left), heat kernel embedding network (center), wave embedding network (right).}\label{fig:circles}
\end{figure}

\section{Examples with Heterogeneous Dimensionality}

\subsection{Plane with Holes}
We examine the behavior of waves in a porous medium.  Figure \ref{fig:planeWithHoles} shows that the wave equation travels more quickly through the bridges (the wave speeds up while in
a bottleneck), and gets caught in the intersections.  preserves the topology of the data and emphasizes the holes.  

\begin{figure}[!h]
\begin{tabular}{ccc}
\includegraphics[height=.25\textwidth]{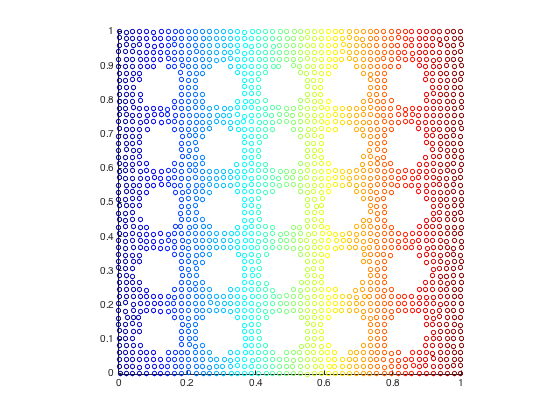} & 
\includegraphics[height=.25\textwidth]{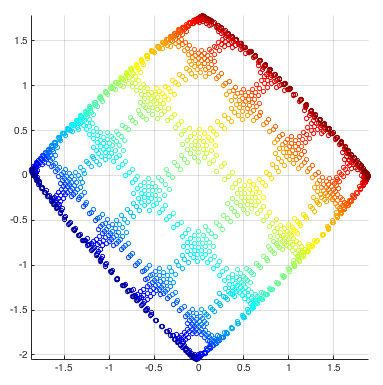} &
\includegraphics[height=.25\textwidth]{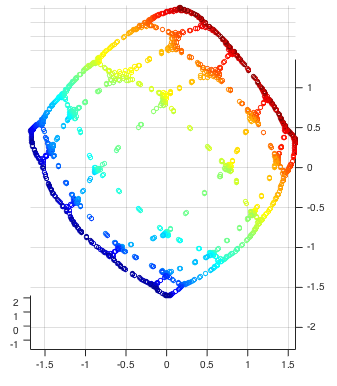} 
\end{tabular}
\captionsetup{width=0.95\textwidth}
\caption{Original Data (left), heat kernel Embedding (center), and Wave Embedding (right).  We used the minimum distance across 10 starting positions.}\label{fig:planeWithHoles}
\end{figure}

\subsection{Union of Manifolds}
Another interesting property of the wave equation is that the change in position $u_t$ undergoes a dramatic change whenever the dimensionality of the manifold changes: waves
are suddenly forced into a very narrow channel or -- going the other direction -- are suddenly evolving in many different directions. 
We demonstrate this first in Figure \ref{fig:6D1D}.  The data consists of two six-dimensional spheres, connected by a one-dimensional line.  The low frequency eigenfunctions of the heat kernel travel from one end to the other without much recognition of the varying dimensionality.  However, the wave embedding creates a gap between the bridge and the two spheres, with the variation of the first non-trivial eigenfunction being supported almost entirely on the bridge.

\begin{figure}[!h]
\begin{tabular}{ccc}
\includegraphics[width=.33\textwidth]{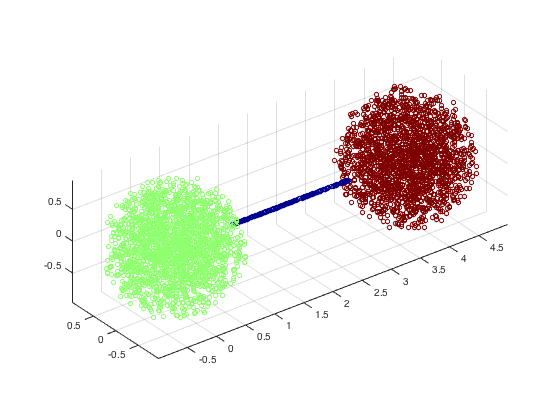} & 
\includegraphics[width=.33\textwidth]{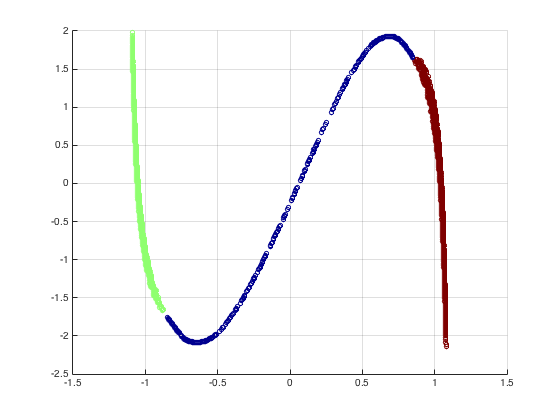} &
\includegraphics[width=.33\textwidth]{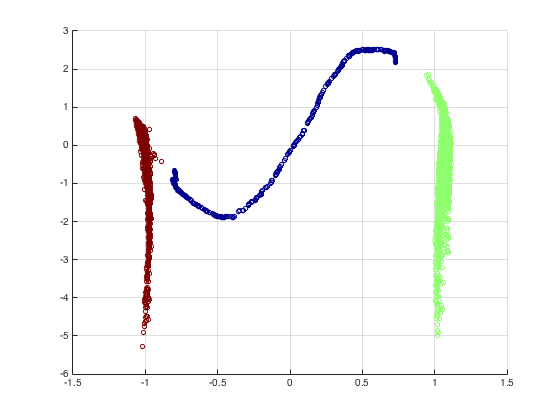} 
\end{tabular}
\caption{6D spheres with 1D bridge (left), heat kernel embedding (center), and wave embedding (right).  For this problem, we use the min distance across 10 starting positions.}\label{fig:6D1D}
\end{figure}

\subsection{Union of Manifolds with different dimensions}
We also consider the same problem with spheres of different dimensions as in Figure \ref{fig:6D3D1D}.  The data consists of a six-dimensional sphere and a three-dimensional sphere connected by a one-dimensional line.  The figure displays the affinity matrices of the points, with the first block representing the six-dimensional sphere, the second block representing the three-dimensional sphere, and the third small block for the bridge.  Notice that, in the heat kernel affinity, the bridge has affinity to more points in the lower dimensional sphere than the higher dimensional sphere.  Also notice that the wave embedding separates the six-dimensional sphere much further from the bridge than the three-dimensional sphere.\\
 
\begin{figure}[!h]
\begin{tabular}{cc}
\includegraphics[width=.48\textwidth]{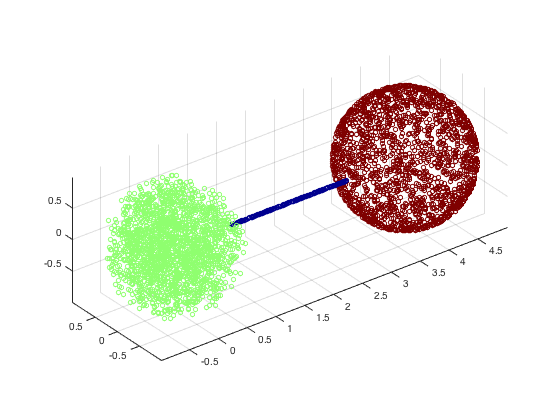} &
\includegraphics[width=.48\textwidth]{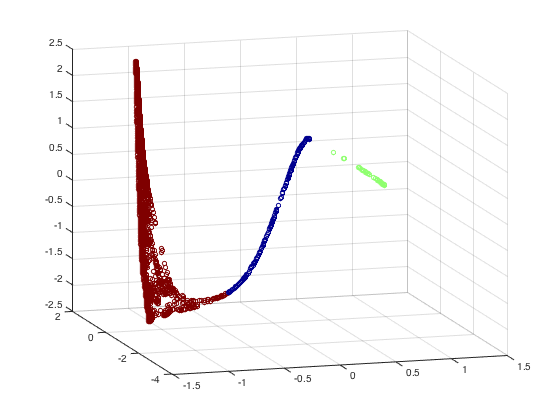} \\
\includegraphics[width=.48\textwidth]{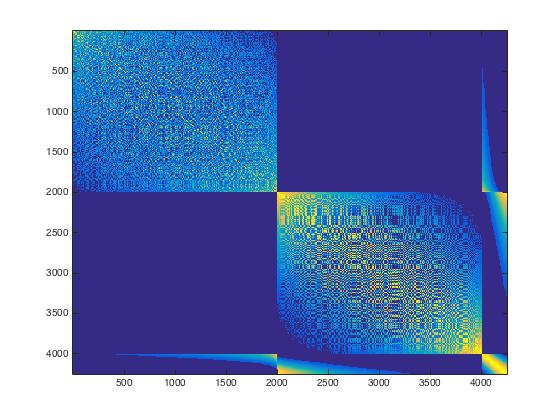} & 
\includegraphics[width=.48\textwidth]{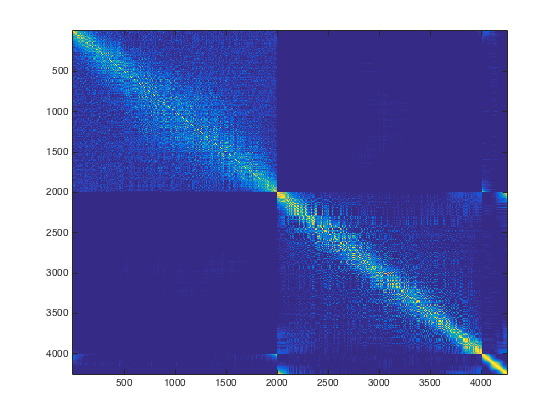} 
\end{tabular}
\captionsetup{width=0.95\textwidth}
\caption{6D sphere in cluster 1 and 3D sphere in cluster 2 with 1D bridge (top left), wave embedding (top right), heat kernel affinity matrix (bottom left), and wave affinity matrix (bottom right).  For this problem, we use the min distance across 10 starting positions.}\label{fig:6D3D1D}
\end{figure}

Finally, we consider two six-dimensional spheres connected via a two-dimensional bridge in Figure \ref{fig:6D2D}.  Specifically, we examine the local affinities of several points on the bridge.  Notice that, for the wave equation, the affinities of points on the bridge are far from isotropic and clearly distinguish the direction the wave is traveling between the two spheres.  Moveover, points on the bridge near the spheres have much lower affinity to points on the sphere than their heat kernel counterparts.  
\begin{figure}[!h]
\begin{tabular}{cc}
\includegraphics[width=.48\textwidth]{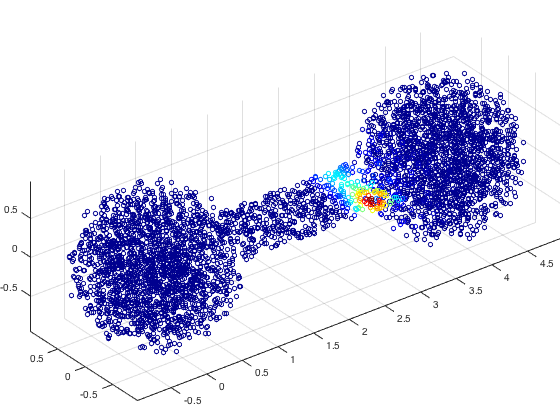} & 
\includegraphics[width=.48\textwidth]{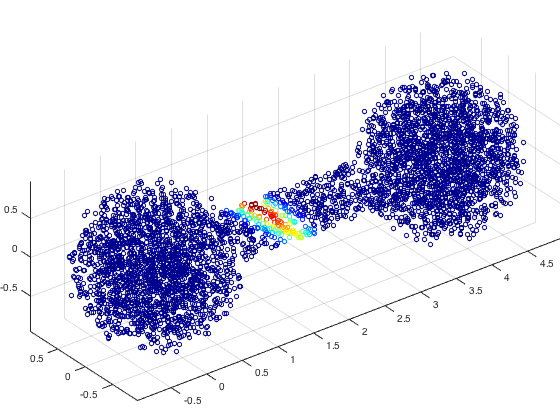} \\ 
\includegraphics[width=.48\textwidth]{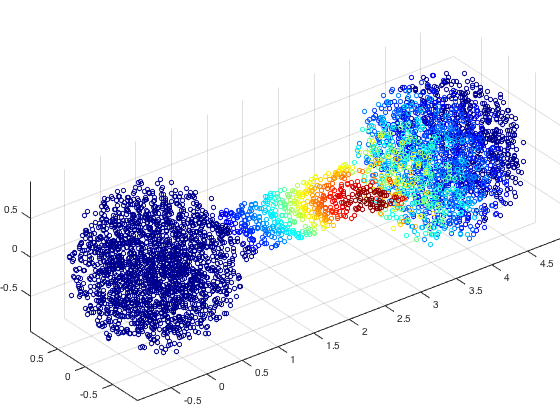} & 
\includegraphics[width=.48\textwidth]{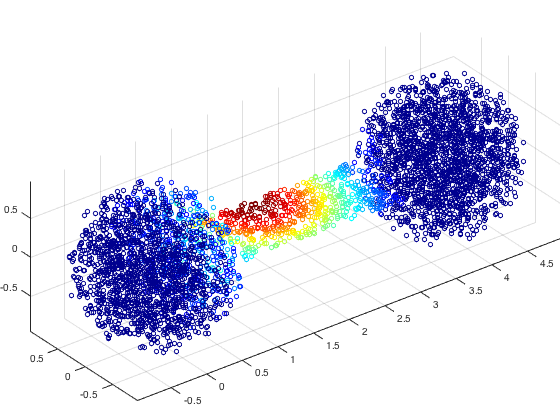}  
\end{tabular}
\captionsetup{width=0.95\textwidth}
\caption{Neighborhoods of chosen points in 6D spheres--2D bridge example for wave embedding using averaging across 10 starting positions (top) and heat kernel affinity (bottom).}\label{fig:6D2D}
\end{figure}

\section{Comments and Remarks}

\subsection{Other partial differential equations}
Spectral echolocation has two novel components:
\begin{enumerate}
\item the evolution of a dynamical system on an existing weighted graph 
\item and the construction of a refined metric using information coming from the behavior of the dynamical system.
\end{enumerate}
Our current presentation had its focus mainly on the case where the dynamical system is given by the wave equation,
however, it is not restricted to that. Let us quickly consider a general linear partial differential equation of the type
$$ \frac{\partial}{\partial t} u(t,x) = D u(t,x) \qquad \mbox{on}~\mathbb{R},$$
where $D$ is an arbitrary differential operator. The Fourier transform in the space
variable yields a separation of frequencies 
$$ \frac{\partial}{\partial t} \widehat{u}(t,\xi) = i p(D, \xi) \widehat{u}(t, \xi),$$
where $p(D,\xi)$ is the symbol of the differential operator at frequency $\xi$.
This is a simple ordinary differential equation whose solution can be written down
explicitly as
$$ \widehat{u}(t, \xi) = e^{i t p(D, \xi)} \widehat{u}(0, \xi)$$
and taking the Fourier transform again allows us to write the solution as
$$  u(t,x) = \int_{\mathbb{R}}{ e^{i x \cdot \xi + i t p(D, \xi)} \widehat{u}(0, \xi) d\xi}.$$
Differential equations for which this scheme is applicable include the classical heat equation
($D = \Delta$) but also variants that include convolution with a sufficiently nice potential ($Du = \Delta u + V * u$), the Airy equation ($D = \partial_{xxx}$) and, more
generally, any sufficiently regular pseudo-differential operator (for example $\sqrt{-\Delta + c^2}$).
The crucial insight is that the abstract formulation via the Fourier transform has a direct
analogue on weighted graphs: more precisely, given eigenfunctions $\phi_1, \dots, \phi_N$
associated to eigenvalues $\lambda_1, \dots, \lambda_N$, the natural `frequency' associated
to $\phi_k$ is, of course, $\lambda_k$ and we may define the solution of
$$ \frac{\partial}{\partial t} u(t,x) = D u(t,x)$$
in the same way via
$$ u(t,x) = \sum_{k=1}^{N}{ e^{  p(\lambda) t} \left\langle u(0,x), \phi_k \right\rangle_{L^2} \phi_k(x)}.$$

\begin{figure}[h!]
\begin{tabular}{ccc}
\includegraphics[width=.32\textwidth]{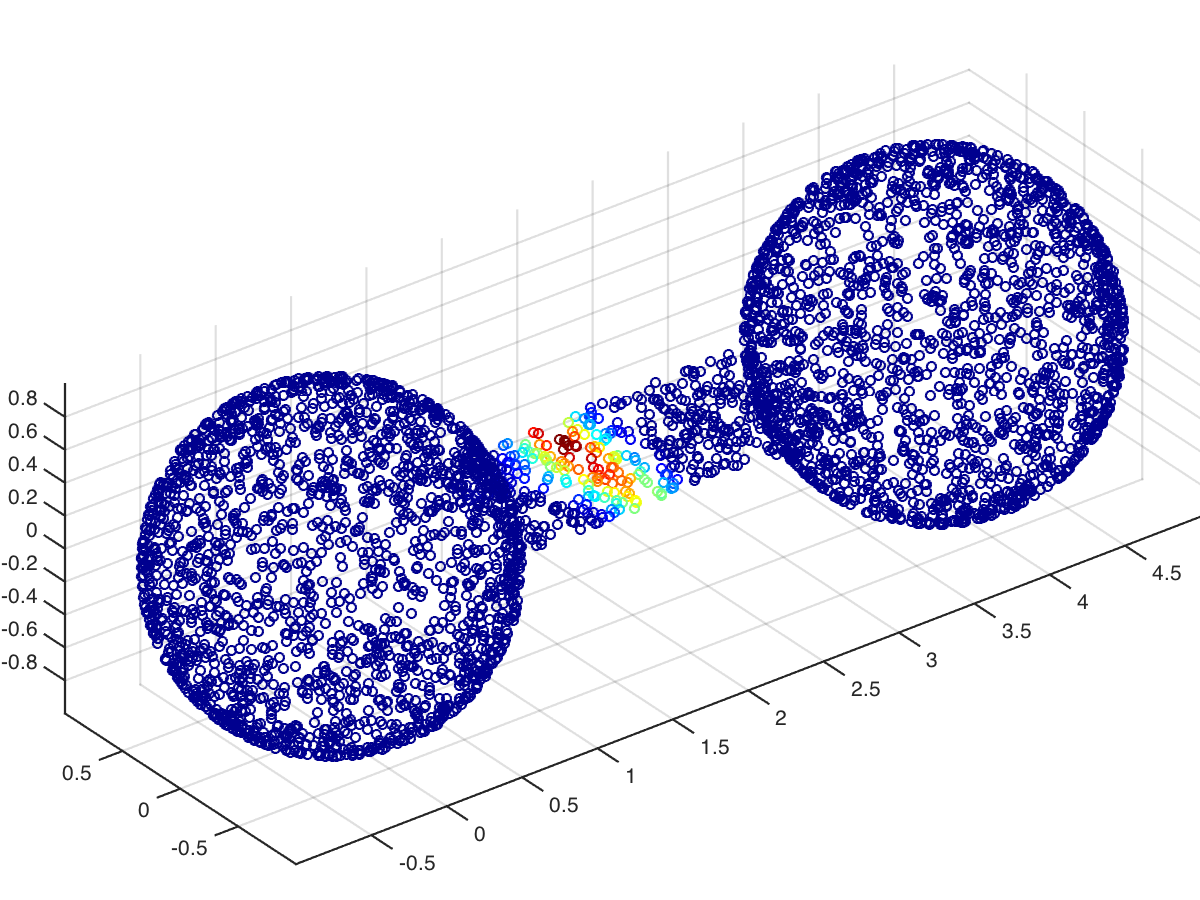} & 
\includegraphics[width=.32\textwidth]{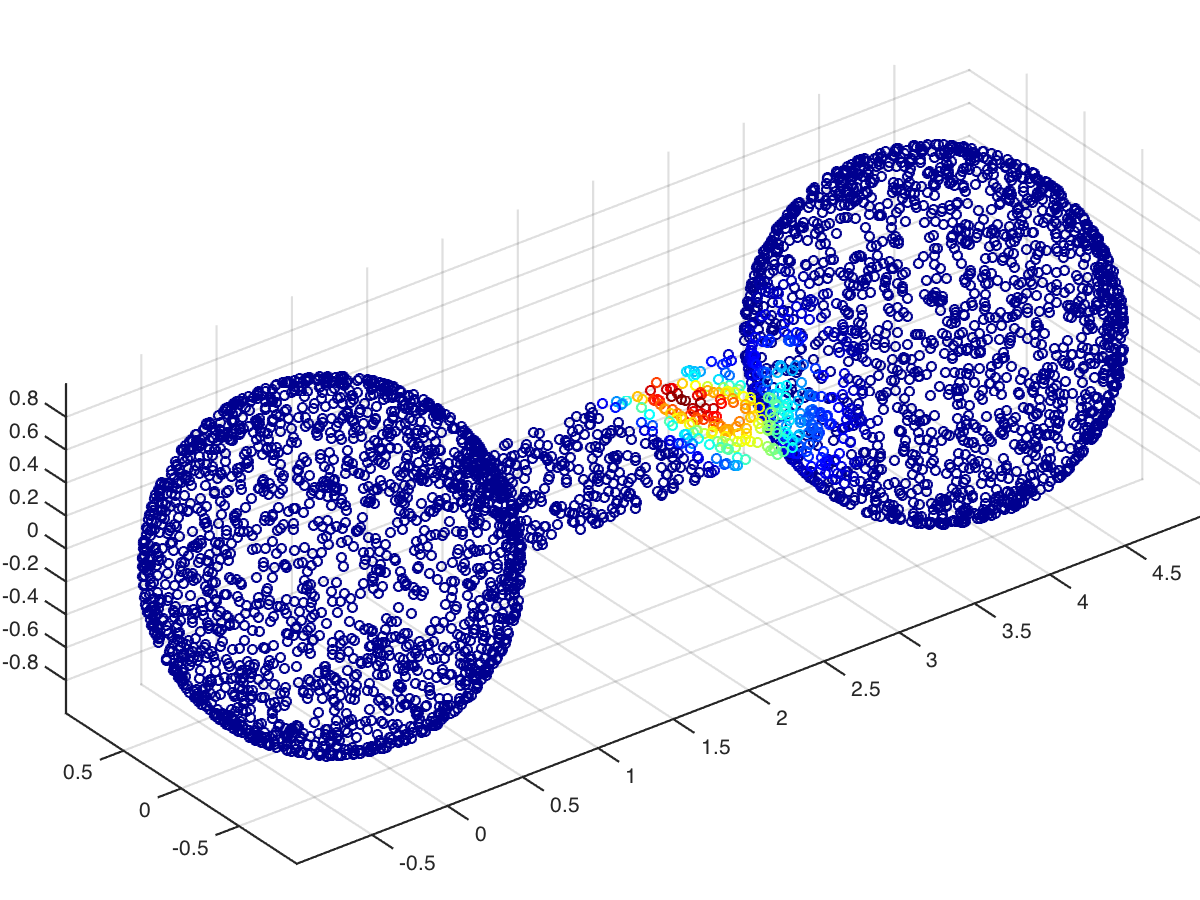} & 
\includegraphics[width=.32\textwidth]{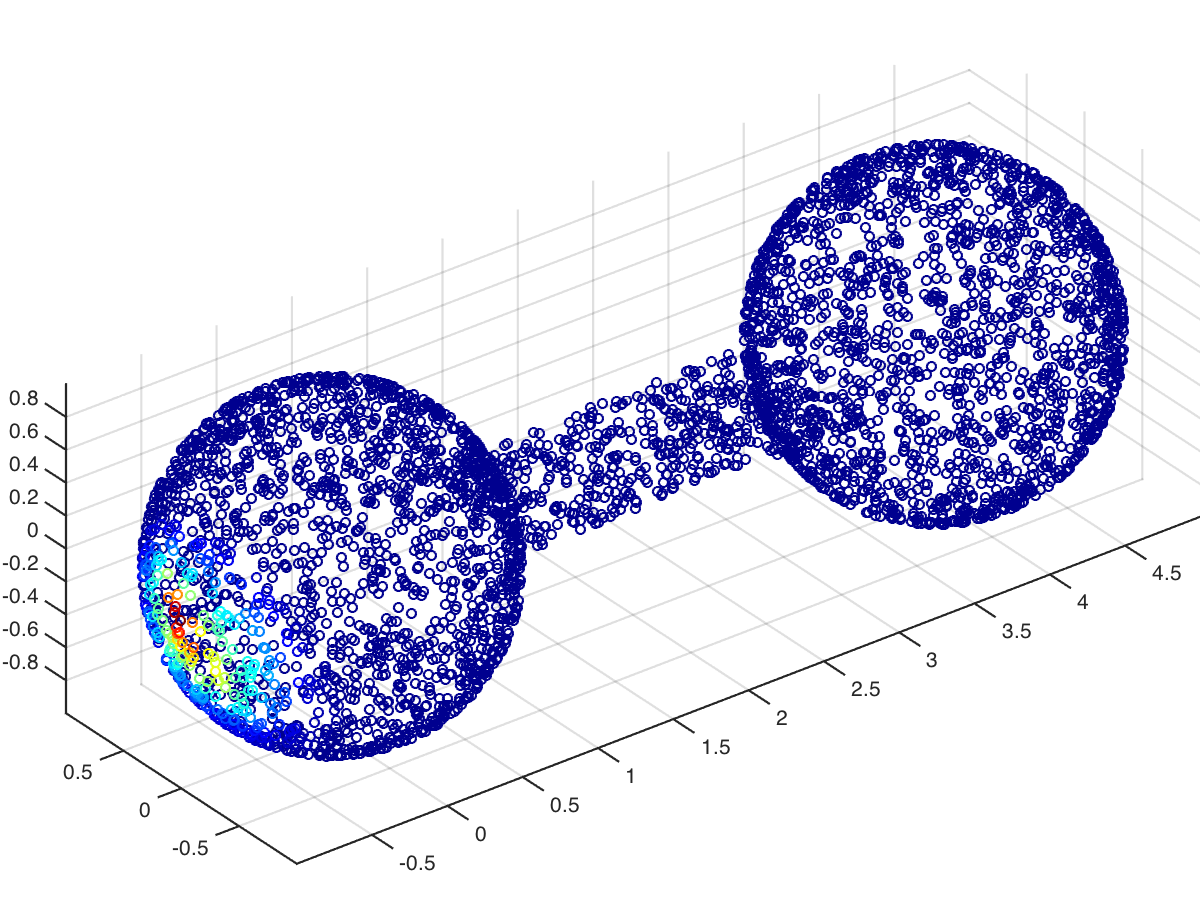} \\
\includegraphics[width=.32\textwidth]{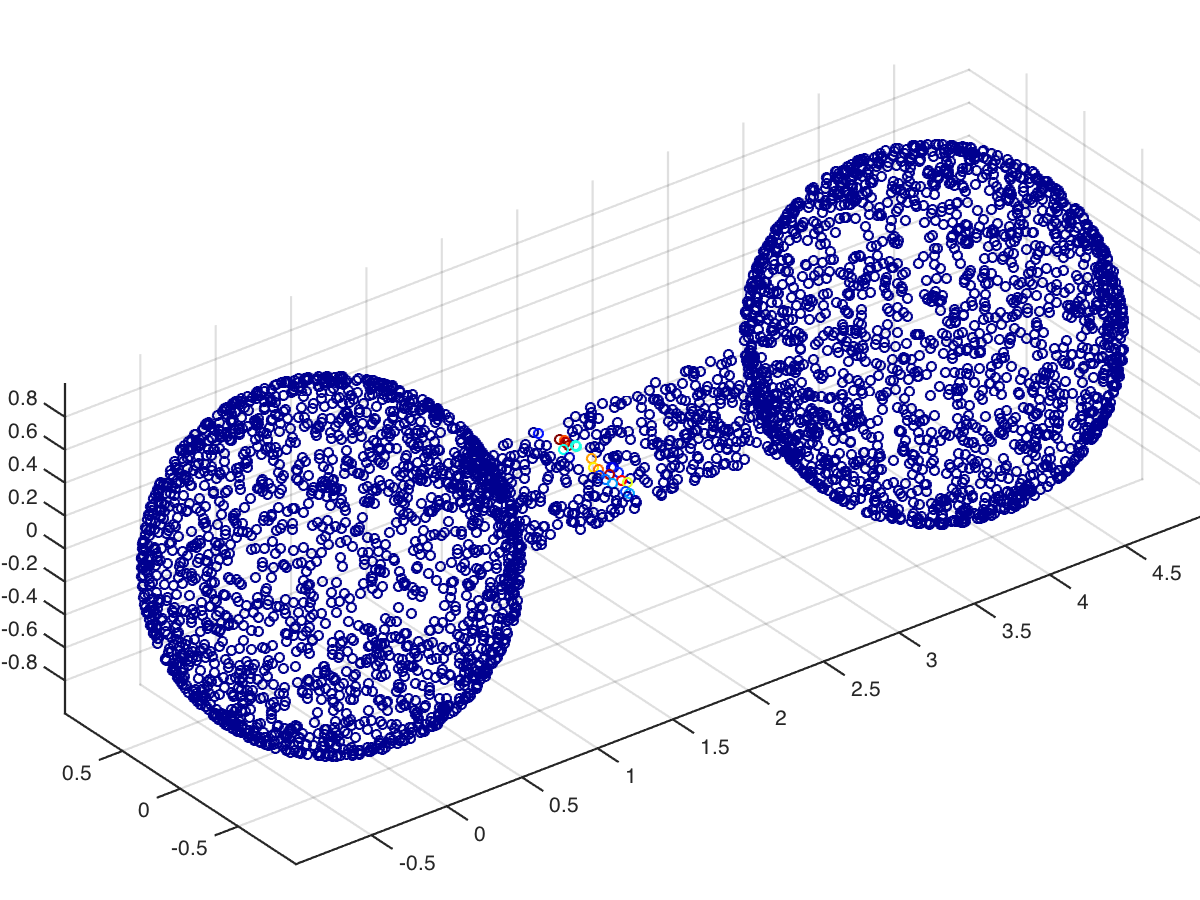} & 
\includegraphics[width=.32\textwidth]{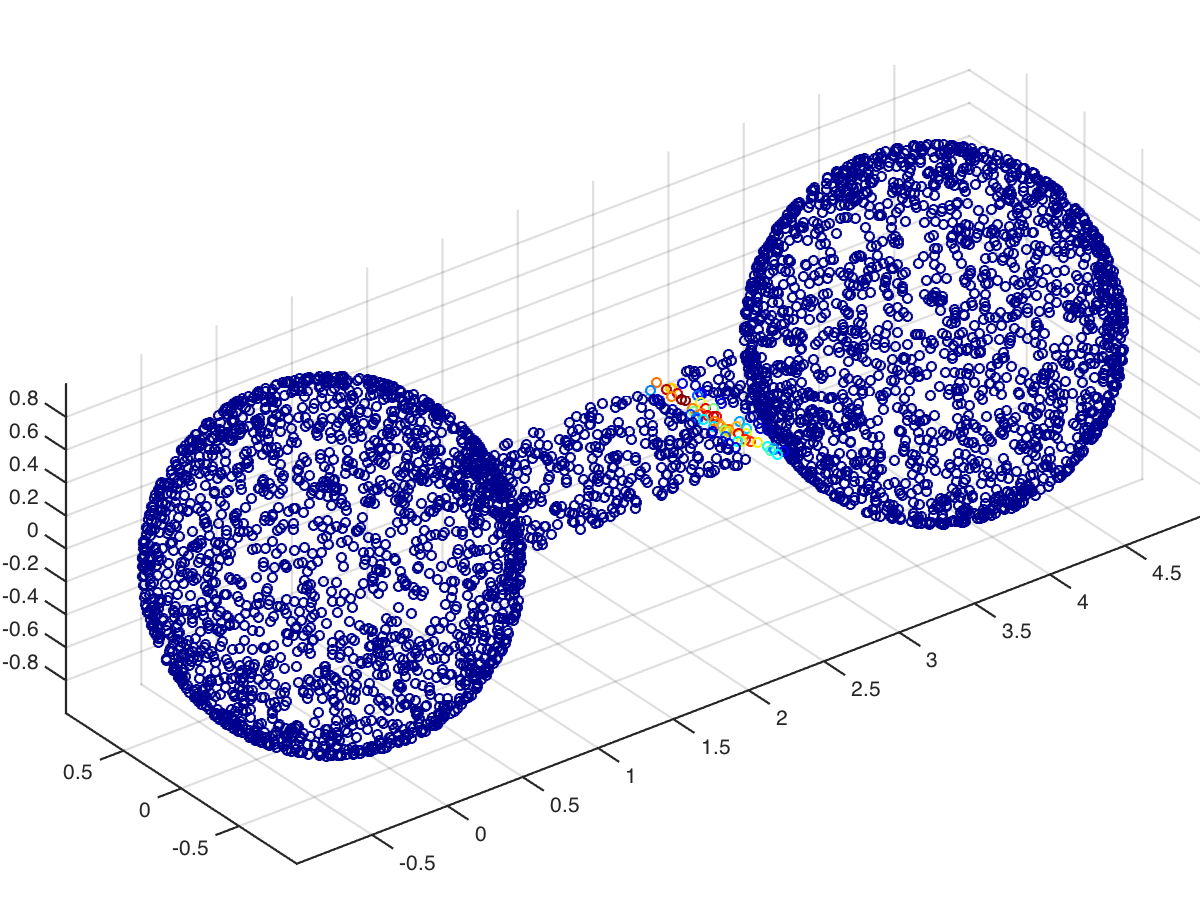} & 
\includegraphics[width=.32\textwidth]{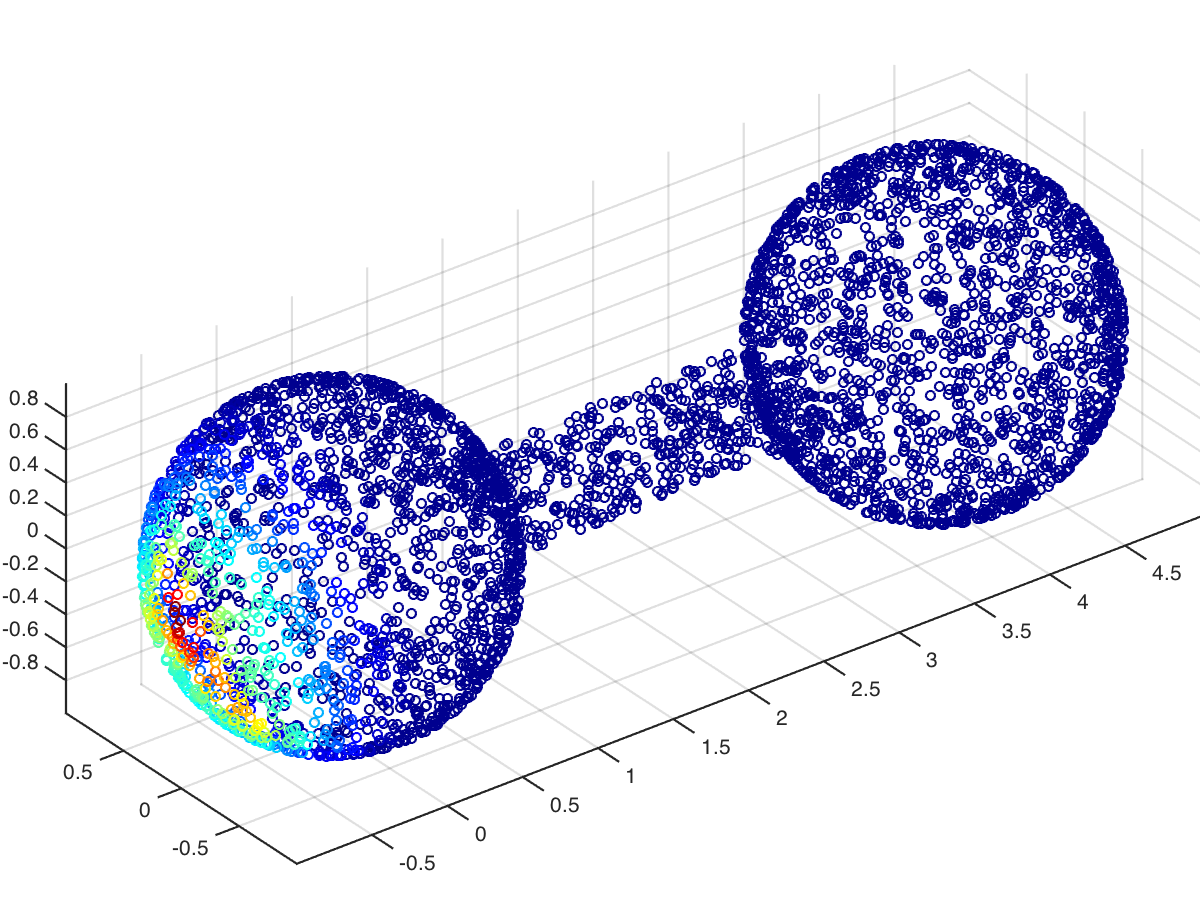} \\
\includegraphics[width=.32\textwidth]{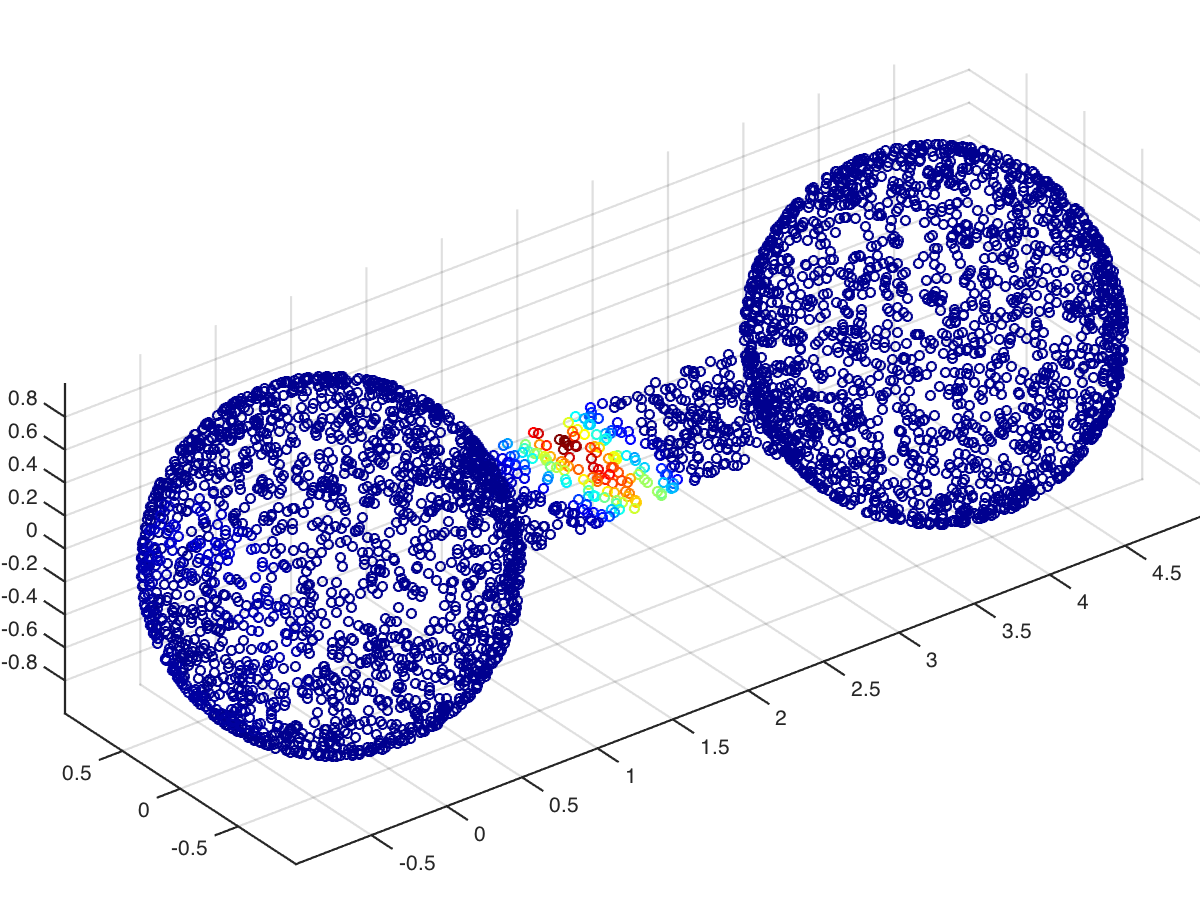} & 
\includegraphics[width=.32\textwidth]{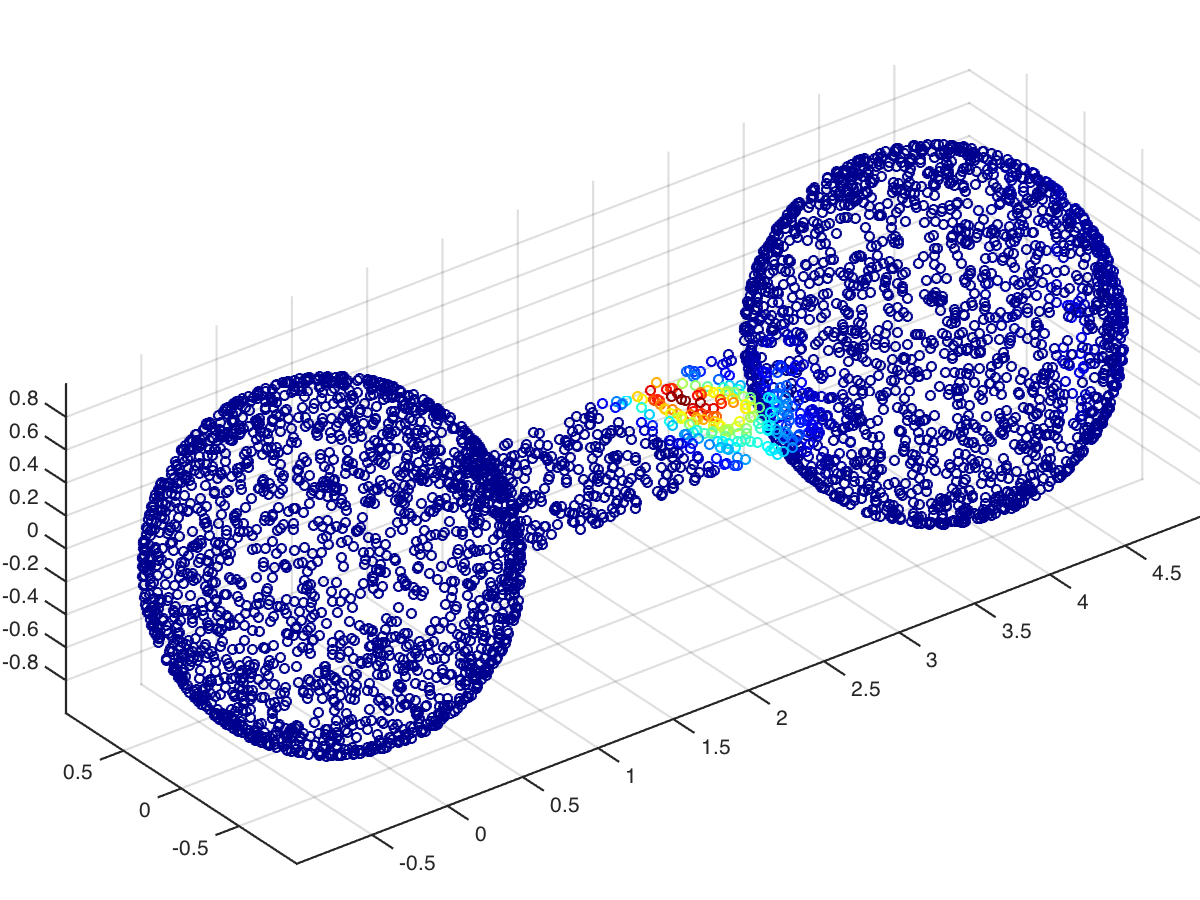} & 
\includegraphics[width=.32\textwidth]{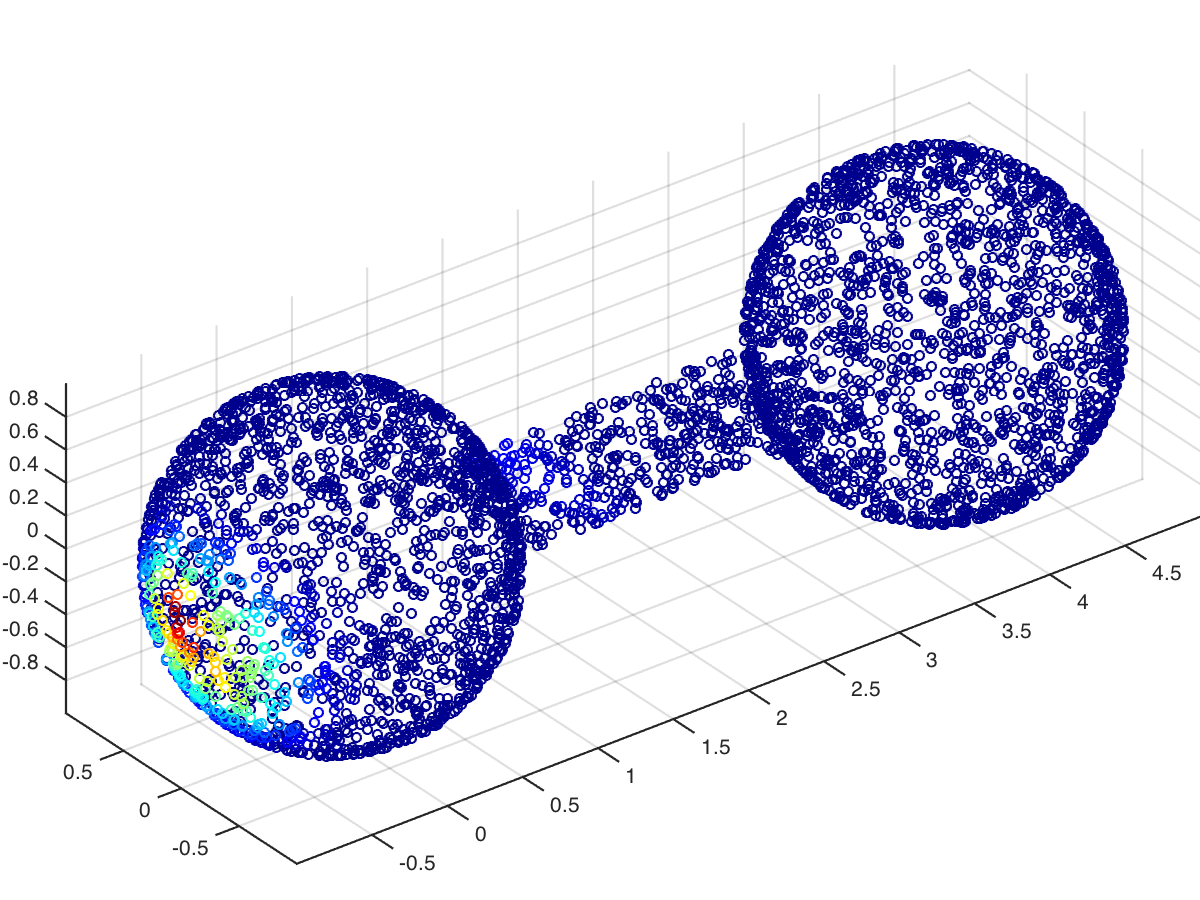} \\
\includegraphics[width=.32\textwidth]{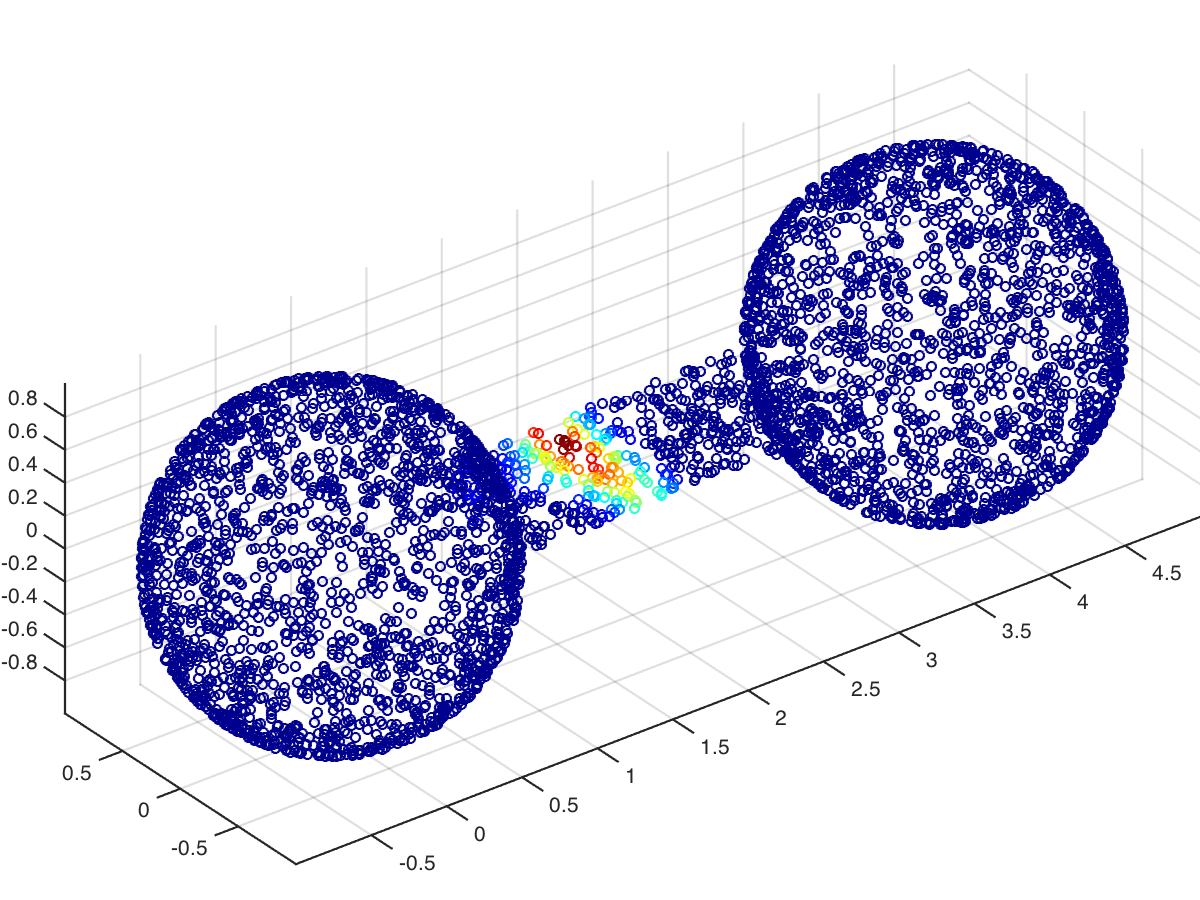} & 
\includegraphics[width=.32\textwidth]{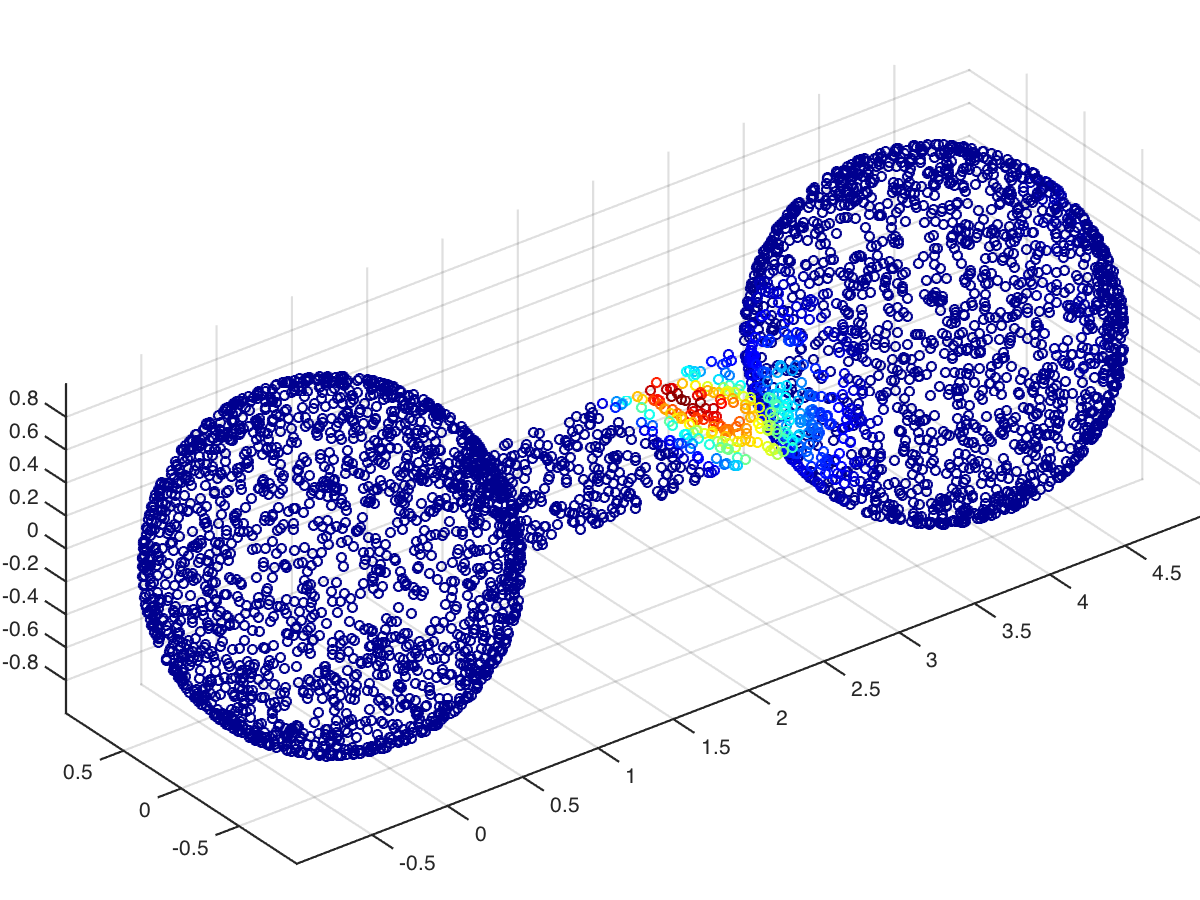} & 
\includegraphics[width=.32\textwidth]{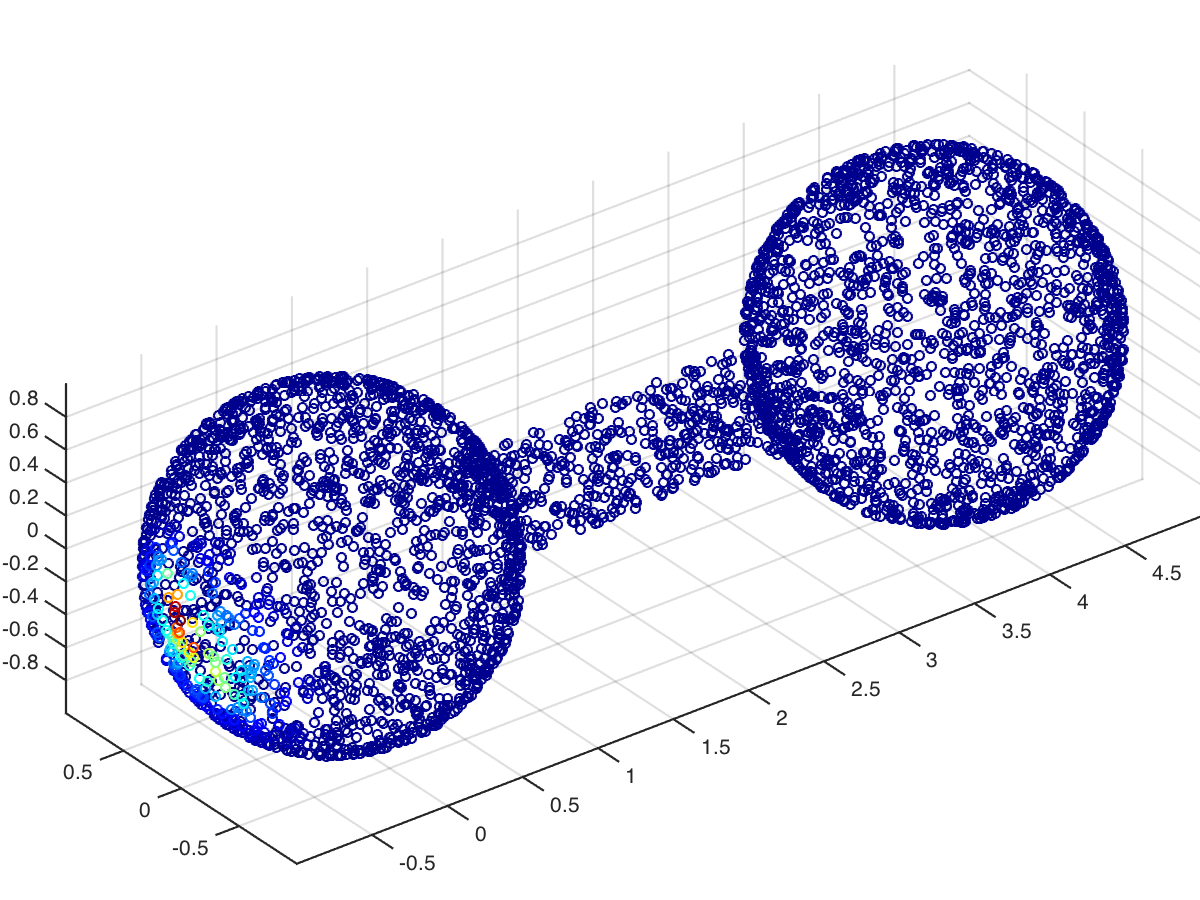} 
\end{tabular}
\caption{Wave (top), Heat ($2^{nd}$ row), Airy ($3^{rd}$ row), and Schrodinger (bottom).}
\end{figure}

Natural `symbols' include heat $p(\lambda) = - \lambda^2$, Airy $p(\lambda) = i \lambda^3$ or Schrodinger $p(\lambda) = i \lambda$. Naturally, the same analysis goes through for equations of the type $u_{tt} = Du$ and our analysis of the attenuated wave equation above follows that scheme. The analysis of partial differential equations on graphs is still in its infancy and our original motivation for using the wave equation is a number of desirable properties that seem uniquely suited for the task at hand: no dissipation of energy and finite speed of propagation. Numerical examples show that different symbols $p(\lambda)$ can induce very similar
neighborhoods: we believe that this merits further study; in particular, a thorough theoretical analysis of the
proposed family of algorithms is highly desired.

\subsection{Special case: Heat equation} We want to emphasize that our approach is novel even when we
chose to emulate the classical heat propagation. This method can outperform the classical (unrefined)
embedding via Laplacian eigenmaps even in relatively simple toy examples: we consider the classical 2D dumbbell example in Figure \ref{fig:dumbbell2}.

\begin{figure}[h!]
\begin{tabular}{c}
\includegraphics[width=.5\textwidth]{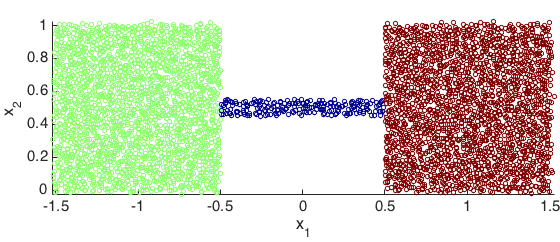} 
\end{tabular}
\caption{The dumbbell domain in our experiment.}\label{fig:dumbbell2}
\end{figure}

This example has a small Cheeger constant due to the bottleneck, which means the first non-trivial eigenfunction will be essentially constant on the boxes and change rapidly on the bottleneck. This classical examples
illustrates well how the first nontrivial eigenfunction can be used as  a classifier and the classical Laplacian
eigenmap works spectacularly well without any further modifications.

\begin{figure}[h!]
\begin{tabular}{c}
\includegraphics[width=.6\textwidth]{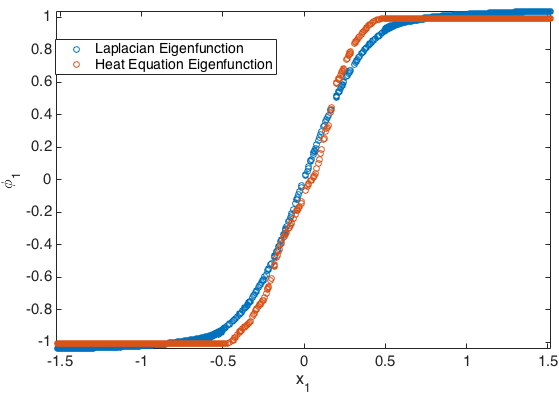}
\end{tabular}
\caption{The values of the first eigenfunction for classical Laplacian eigenmaps (blue) vs. the refined heat metric (blue).}\label{fig:dumbbell}
\end{figure}

Figure \ref{fig:dumbbell} compares the eigenfunction $\phi_1$ of the Laplacian compared to the first non-trivial eigenfunction of the heat equation distance matrix.  We observe that the refined heat metric is a much better approximation to the function
\[f(x) =  \begin{cases} 
      -1 & x_1\leq -0.5 \\
      2x  & -0.5< x_1< 0.5 \\
      1 & x_1\geq 0.5
   \end{cases}
\]
than $\phi_1$ and allows for a more accurate reconstruction of the bridge. We also observe that the nontrivial eigenfunction is essentially and to a remarkable degree constant on the two clusters which further increases its value as a classifier.

\section{Conclusions}
\textit{Summary.} We have presented a new pre-processing technique for classical dimensionality reduction techniques based on spectral
methods.  The underlying new idea comes in three parts: (1) if one computes eigenfunctions of the Laplacian, then one might
just as well use them so simulate the evolution of a partial differential equation on the existing weighted graph, (2) especially
for physically meaningful dynamical systems such as the wave equation one would expect points with high affinity to behave
similarly throughout time and (3) this motivates the construction of a refined metric extracting information coming from the behavior of the dynamical system. \\

\textit{The wave equation.} We were originally motivated by a series of desirable properties of the wave equation on $\mathbb{R}^n$: preservation of regularity and finite speed of propagation. Recall that one of the fundamental differences between the heat equation and the wave equation is that solutions of the heat equation experience an instanteneous gain in smoothness while the wave equation merely preserves the smoothness of the initial datum (and sometimes not even that). Our main point is to show
that this is not arbitrary but due to physical phenomena whose counterparts in the world of data can provide a refined measurement: \textit{the lack of regularity can be helpful}! 
However, as we have shown, there are similar effects for most other partial differential equations and 
theoretical justifications on a precise enough level that they would distinguish between various dynamical
systems are still missing -- we believe this to be a fascinating open problem.\\

\textit{Refined metrics.}   Similarily, our way of refining metrics, either by taking the minimum or by compiling an average, is motivated by considerations (see also \cite{stein}) that are
not specifically tuned to our use of partial differential equations -- another fascinating open question is whether there is
a more natural and attuned way of extracting information.\\

\textbf{Acknowledgement.} The authors are grateful to Raphy Coifman for a series of fruitful discussions
and helpful suggestions. A.C. is supported by an NSF Postdoctoral Fellowship \#1402254, S.S. is supported by an
AMS Simons Travel Grant and INET Grant \#INO15-00038.

\end{document}